\documentclass{article}
\usepackage{ijcai22}

\usepackage{times}

\usepackage{soul}
\usepackage{url}
\usepackage[hidelinks]{hyperref}
\usepackage[utf8]{inputenc}
\usepackage[small]{caption}
\usepackage{graphicx}
\graphicspath{{figures/}}
\usepackage{amsmath}
\usepackage{booktabs}
\usepackage[ruled,linesnumbered,vlined]{algorithm2e}
\usepackage{amsmath,amssymb,amsfonts,amsthm}
\usepackage{latexsym}
\usepackage{thm-restate}
\usepackage{ifdraft}
\usepackage{tikz}
\usepackage{upgreek}
\usepackage{subcaption}  
\usepackage{array}  
\usepackage{booktabs}  

\usepackage[inline]{enumitem}
\newlist{romanenumerate*}{enumerate*}{1}
\setlist[romanenumerate*]{label=(\textit{\roman*})}
\newlist{romanenumerate}{enumerate}{1}
\setlist[romanenumerate]{label=(\textit{\roman*})}

\usepackage{mymacros}

\newtheorem{theorem}{Theorem}

\newtheorem{definition}{Definition}

\newtheorem{guarantees}{Guarantees}

\title{Markov Abstractions for PAC Reinforcement Learning in\\ Non-Markov
Decision Processes}

\author{
  Alessandro Ronca\and
  Gabriel Paludo Licks\and
  Giuseppe De Giacomo\\
  \affiliations
  DIAG, Sapienza University of Rome, Italy
  \emails
  \{ronca, licks, degiacomo\}@diag.uniroma1.it
}

\begin{document}

\maketitle

\begin{abstract}
  Our work aims at developing reinforcement learning algorithms that do not rely
  on the Markov assumption. We consider the class of Non-Markov Decision
  Processes where histories can be abstracted into a finite set of states while
  preserving the dynamics. We call it a \emph{Markov abstraction} since it
  induces a Markov Decision Process over a set of states that encode the
  non-Markov dynamics. This phenomenon underlies the recently introduced Regular
  Decision
  Processes (as well as POMDPs where only a finite number of belief states is
  reachable). In all such kinds of decision process, an agent that uses a Markov
  abstraction can rely on the Markov property to achieve optimal behaviour. We
  show that Markov abstractions can be learned during reinforcement learning.
  Our approach combines automata learning and classic reinforcement learning.
  For these two tasks, standard algorithms can be employed. We show that our
  approach has PAC guarantees when the employed
  algorithms have PAC guarantees, and we also provide an experimental
  evaluation.
\end{abstract}

\section{Introduction}
\label{sec:introduction}

In the classic setting of Reinforcement Learning (RL), the agent is
provided with the current \emph{state} of the environment \cite{suttonbarto}.
States are a useful abstraction for agents, since predictions and decisions can
be made according to the current state.
This is RL under the Markov assumption, or \emph{Markov RL}.
Here we focus on the more realistic \emph{non-Markov RL} setting
\cite{hutter2009feature,brafman2019rdp,toroicarte2019learning,ronca2021efficient}
where the agent is not given the current state, but can observe what happens in
response to its actions. 
The agent can still try to regain the useful abstraction of states. However, now
the abstraction has to be learned, as a map from every \emph{history} observed
by the agent to some \emph{state} invented by the agent \cite{hutter2009feature}.

We propose RL agents that learn a specific kind of abstraction called
\emph{Markov abstraction} as part of the overall learning process.
Our approach combines automata learning and Markov RL in a modular manner, with
Markov abstractions acting as an interface between the two modules. A key aspect
of our contribution is to show how the sequence of intermediate automata built
during learning induce partial Markov abstractions that can be readily used to
guide exploration or exploitation. 
Our approach is Probably Approximately Correct (PAC), cf.\
\cite{kearns1994introduction}, whenever the same holds for the
employed automata and Markov RL algorithms.

The idea of solving non-Markov tasks by introducing a Markovian state space 
can already be found in \cite{bacchus1996rewarding} and
more recently in \cite{brafman2018ltlf,toroicarte2018reward}. 
RL in this setting has been
considered \cite{toroicarte2018reward,degiacomo2019bolts,gaon2019rl,xu2020jirp}.
This setting is simpler than ours since transitions are still assumed to be
Markov. The setting where both transitions
and rewards are non-Markov has been considered in
\cite{toroicarte2019learning,brafman2019rdp}, with RL studied in
\cite{toroicarte2019learning,abadi2020learning,ronca2021efficient}.
Such RL techniques are based on automata learning. The approach in
\cite{ronca2021efficient} comes with PAC guarantees, as opposed
to the others which do not.
Our approach extends the techniques in \cite{ronca2021efficient} in order to use
the learned automata not only to construct the final policy, but also to
guide exploration.

Abstractions from histories to states have been studied in 
\cite{hutter2009feature,maillard2011selecting,veness2011montecarlo,nguyen2013competing,lattimore2013sample,hutter2016extreme,majeed2018qlearning}.
\cite{hutter2009feature}
introduces the idea of abstractions from histories to states.  Its algorithmic
solution, as well as the one in \cite{veness2011montecarlo}, is less general
than ours since it assumes a bound on the length of the histories to consider.
This corresponds to a subclass of automata.
\cite{maillard2011selecting} provides a technique to select an abstraction from
a given finite set of candidate abstractions; instead, we consider an infinite
set of abstractions.
\cite{nguyen2013competing}
considers a set of abstractions without committing to a specific way of
representing them. As a consequence, they are not able to take advantage of
specific properties of the chosen representation formalism, in the algorithm nor
in the analysis. On the contrary, we choose automata, which allows us to take
advantage of existing automata-learning techniques, and in particular of their
properties such as the incremental construction.
\cite{lattimore2013sample} studies non-Markov RL in the case
where the non-Markov decision process belongs to a compact class. Their results
do not apply to our case because the class of decision processes admitting a
Markov abstraction is not compact.
\cite{majeed2018qlearning} studies Q-learning with abstractions, but it assumes
that abstractions are given.
\cite{hutter2016extreme} provides a result that we use in Section~3; however,
none of their abstractions is required to preserve the dynamics, as we
require for our Markov abstractions. Even in the case called `exact state
aggregation', their abstractions are only required to preserve rewards, and not
to preserve observations. In this setting, it is unclear whether automata
techniques apply.

Proofs and experimental details are given in the appendix.

\section{Preliminaries}

For $x$ and $z$ strings, $xz$ denotes their concatenation.
For $\Sigma$ and $\Gamma$ alphabets, $\Sigma\Gamma$ denotes the set of all
strings $\sigma\gamma$ with $\sigma \in \Sigma$ and $\gamma \in \Gamma$.
For $f$ and $g$ functions, $fg$ denotes their composition.
We write $f: X \leadsto Y$ to denote a function $f: X \times Y \to [0,1]$ that
defines a probability distribution $f(\cdot|x)$ over $Y$ for every $x \in X$.

\paragraph{Non-Markov Decision Processes.}
A \emph{Non-Markov Decision Process} (NMDP), cf.\ \cite{brafman2019rdp},
is a tuple 
$\mathcal{P} = \langle A, O, R, \terminationsymbol, \transitionfunc, \rewardfunc
\rangle$ with components defined as follows.
$A$ is a finite set of \emph{actions},
$O$ is a finite set of \emph{observations},
$R \subseteq \mathbb{R}_{\geq 0}$ is a finite set of non-negative 
\emph{rewards}, 
$\terminationsymbol$ is a special symbol that denotes \emph{episode
termination}.
Let the elements of $\mathcal{H} = (AOR)^*$ be called \emph{histories}, and
let the elements of $\mathcal{E} = (AOR)^*A\terminationsymbol$ be called
\emph{episodes}.
Then,
$\transitionfunc: \mathcal{H} \times A \leadsto (O \cup \{ \terminationsymbol
\})$ is the \emph{transition function}, and
$\rewardfunc: \mathcal{H} \times A \times O \leadsto \rewardValues$ is
the \emph{reward function}.
The transition and reward functions can be combined 
into the \emph{dynamics function} 
$\dynfunc: \mathcal{H} \times A \leadsto (OR \cup \{ \terminationsymbol \})$,
which describes the probability to observe next a certain pair of
observation and reward, or termination, given a certain history and action.
Namely, 
$\dynfunc(or|h,a) = \transitionfunc(o|h,a) \cdot \rewardfunc(r|h,a,o)$ and
$\dynfunc(\terminationsymbol|h,a) = \transitionfunc(\terminationsymbol|h,a)$.
We often write an NDMP directly as
$\langle A, O, R, \terminationsymbol, \dynfunc \rangle$.
A \emph{policy} is a function $\pi: \mathcal{H} \leadsto A$.
The \emph{uniform policy} $\upolicy$ is the policy defined as $\upolicy(a|h) =
1/|A|$ for every $a$ and every $h$.
The \emph{dynamics of $\mathcal{P}$ under a policy $\pi$} describe
the probability of an episode remainder, given the history so far, when
actions are chosen according to a policy $\pi$; it can be recursively computed
as $\dynfunc_\pi(aore | h) = \pi(a|h) \cdot \dynfunc(or|h,a) \cdot
\dynfunc_\pi(e | haor)$, with base case
$\dynfunc_\pi(a\terminationsymbol | h) = \pi(a|h) \cdot
\dynfunc(\terminationsymbol|h,a)$.
%
Since we study episodic reinforcement learning,
we require episodes to terminate with probability one, i.e.,
$\sum_{e \in \mathcal{E}} \dynfunc_\pi(e|\varepsilon) = 1$ for every policy
$\pi$.\footnote{A constant probability $p$ of terminating at every step amounts
  to a discount factor of $1-p$, see \cite{puterman1994markov}.}
This requirement ensures that the following value functions take
a finite value.
The \emph{value of a policy $\pi$ given a history $h$},
written $\valfunc_\pi(h)$, is the expected sum of future rewards when
actions are chosen according to $\pi$ given that the history so far is $h$; it
can be recursively computed as
$\valfunc_\pi(h) = \sum_{aor}
\pi(a|h) \cdot \dynfunc(o,r|h,a) \cdot (r + \valfunc_\pi(haor))$.
The \emph{optimal value given a history $h$} is 
$\valfunc_*(h) = \max_\pi \valfunc_\pi(h)$,
which can be expressed without reference to any policy as 
$\valfunc_*(h) = \max_a \left(
  \sum_{or} \dynfunc(o,r|h,a) \cdot (r +
\valfunc_*(haor)) \right)$.
The \emph{value of an action $a$ under a policy $\pi$ given a history $h$},
written $\avalfunc_\pi(h,a)$, is the expected sum of future rewards
when the next action is $a$ and the following actions are chosen according to
$\pi$, given that the history so far is $h$; it is
$\avalfunc_\pi(h,a) = \sum_{or} 
\dynfunc(o,r|h,a) \cdot (r + \valfunc_\pi(haor))$.
The \emph{optimal value of an action $a$ given a history $h$} is 
$\avalfunc_*(h,a) = \max_\pi \avalfunc_\pi(h,a)$, and it can be expressed as 
$\avalfunc_*(h,a) = \sum_{or} 
\dynfunc(o,r|h,a) \cdot (r + \valfunc_*(haor))$.
A policy $\pi$ is \emph{optimal} if 
$\valfunc_\pi(\varepsilon) = \valfunc_*(\varepsilon)$. 
For $\epsilon > 0$, a policy $\pi$ is \emph{$\epsilon$-optimal} if
$\valfunc_\pi(\varepsilon) \geq \valfunc_*(\varepsilon) - \epsilon$.


\paragraph{Markov Decision Processes.}
A \emph{Markov Decision Process (MDP)} 
\cite{bellman1957markovian,puterman1994markov} is a decision process where
the transition and reward functions (and hence the
dynamics function) depend only on the last observation in the history, taken to
be an arbitrary observation when the history is empty.
Thus, an observation is a complete description of the state of affairs, and
it is customarily called a \emph{state} to emphasise this aspect. Hence, we talk
about a set $S$ of states in place of a set $O$ of observations.
All history-dependent functions---e.g., transition and reward functions,
dynamics, value functions, policies---can be seen as taking a single state in
place of a history.

\paragraph{Episodic RL.}
Given a decision process $\rdp$ and a required accuracy $\epsilon > 0$,
Episodic Reinforcement Learning (RL) for $\mathcal{P}$ and $\epsilon$ is the
problem of an agent that
has to learn an $\epsilon$-optimal policy for $\mathcal{P}$ from the data it
collects by interacting with the environment.
The interaction consists in the agent iteratively performing an action and
receiving an observation and a reward in response, until episode termination.
Specifically, at step $i$, the agent performs an action $a_i$,
receiving a pair of an observation $o_i$ and a reward $r_i$, or the termination
symbol $\terminationsymbol$,  
according to the dynamics of the decision process $\mathcal{P}$.
This process generates an episode of the form 
$a_1o_1r_1a_2o_2r_2 \dots a_n \terminationsymbol$. 
The collection of such episodes is the data available to the agent for learning.

\paragraph{Complexity Functions.}
Describing the behaviour of a learning algorithm requires
a way to measure the complexity of the instance $\mathcal{I}$ to be learned, in
terms of a set of numeric parameters that reasonably describe the complexity of
the instance---e.g., number of states of an MDP.
Thus, for every instance $\mathcal{I}$, we have a list 
$\vec{c}_\mathcal{I}$ of values for the parameters.
We use the complexity description to state properties for the algorithm by
means of complexity functions. A \emph{complexity function} $f$ is an
integer-valued non-negative function of the complexity parameters and possibly
of other parameters. Two distinguished parameters are accuracy
$\epsilon > 0$, and probability of failure $\delta \in (0,1)$. 
The dependency of complexity functions on these parameters is of the form
$1/\epsilon$ and $\ln(1/\delta)$.
When writing complexity functions, we will hide the specific choice of
complexity parameters by writing $\mathcal{I}$ instead of $\vec{c}_\mathcal{I}$,
and we will write $\epsilon$ and $\delta$ for $1/\epsilon$ and $\ln(1/\delta)$.
For instance, we will write $f(\mathcal{I},\epsilon,\delta)$ in place of
$f(\vec{c}_\mathcal{I},1/\epsilon,\ln(1/\delta))$.

\paragraph{Automata.}
A \emph{Probabilistic Deterministic Finite Automaton (PDFA)}, cf.\
\cite{balle2014adaptively}, is a tuple 
$\atm = \langle Q, \Sigma, \Gamma, \tau, \lambda, q_0 \rangle$
where: $Q$ is a finite set of states;
$\Sigma$ is a finite alphabet;
$\Gamma$ is a finite set of termination symbols not in $\Sigma$;
$\tau: Q \times \Sigma \to Q$ is the (deterministic) transition function;
$\lambda: Q \times (\Sigma \cup \Gamma) \to [0,1]$ is the
probability function, which defines a probability distribution
$\lambda(\cdot|q)$ over $\Sigma \cup \Gamma$ for every
state $q \in Q$;
$q_0 \in Q$ is the initial state.
The iterated transition function $\tau^*: Q \times \Sigma^* \to Q$  is
recursively defined as
$\tau^*(q,\sigma_1 \dots \sigma_n) = 
\tau^*(\tau(q,\sigma_1), \sigma_2 \dots \sigma_n)$
with base case $\tau^*(q,\varepsilon) = q$; 
furthermore, $\tau^*(w)$ denotes $\tau^*(q_0,w)$.
It is required that,
for every state $q \in Q$, there exists a sequence $\sigma_1, \dots,
\sigma_n$ of symbols of $\Sigma$ such that $\lambda(\tau^*(q,\sigma_1 \dots
\sigma_{i-1}),\sigma_i) > 0$ for every $i \in [1,n]$, and
$\lambda(\tau(q,\sigma_1 \dots \sigma_n),\gamma) > 0$ for $\gamma \in
\Gamma$---to ensure that every string terminates with probability one.
Given a string $x \in \Sigma^*$, the automaton represents the probability
distribution $\atm(\cdot|x)$ over $\Sigma^*\Gamma$ defined
recursively as 
$\atm(\sigma y | x) = \lambda(\tau(x),\sigma) \cdot \atm(y | x \sigma)$ with
base case $\atm(\gamma | x) = \lambda(\tau(x), \gamma)$ for $\gamma \in \Gamma$.

\section{Markov Abstractions}

Operating directly on histories is not desirable.
There are exponentially-many histories in the episode length, and typically each
history is observed few times, which does not allow for computing accurate
statistics.
A solution to this problem is to abstract histories to a reasonably-sized set of
states while preserving the dynamics.
We fix an NMDP 
$\mathcal{P} = \langle A, O, R, \terminationsymbol, \dynfunc \rangle$.
\begin{definition}
  A Markov abstraction over a finite set of states $S$ is a function $\alpha:
  \mathcal{H} \to S$ such that,
  for every two histories $h_1$ and $h_2$,
  $\abstracfunc(h_1) = \abstracfunc(h_2)$ implies
  $\dynfunc(\cdot|h_1,a) = \dynfunc(\cdot|h_2,a)$ for every action $a$.
\end{definition}
Given an abstraction $\alpha$, its repeated application $\alpha^*$ transforms a
given history by replacing observations by the corresponding states as follows:
\begin{align*}
  & \alpha^*(a_1 o_1 r_1 \dots a_n o_n r_n) = s_0 a_1 s_1 r_1 \dots a_n s_n r_n,
  \\
  &\text{where }\, s_i = \alpha(h_i) \text{ and } h_i = a_1 o_1 r_1 \dots
  a_i o_ir_i.
\end{align*}
Now consider the probability $P_\alpha(s_i,r_i|h_{i-1},a_i)$ obtained by
marginalisation as:
\begin{equation*}
  P_\alpha(s_i,r_i|h_{i-1},a_i) = \sum_{o: \alpha(h_{i-1}a_io) = s_i}
  \dynfunc(o,r_i|h_{i-1},a_i).
\end{equation*}
Since the dynamics $\dynfunc(o,r_i|h_{i-1},a_i)$ are the same for every history
mapped to the same state, there is an MDP $\mathcal{M}_\rdp^\alpha$ with dynamics
$\dynfunc^\alpha$ such that 
$P_\alpha(s_ir_i|h_{i-1},a_i) = \dynfunc^\alpha(s_ir_i|\alpha(h_{i-1}),a_i)$.
The \emph{induced MDP} $\mathcal{M}_\rdp^\alpha$ can be solved in place of
$\rdp$.
Indeed, the value of an action in a state is the value of the action in any of
the histories mapped to that state [Hutter, 2016, Theorem~1].
In particular, if $\pi$ is an $\epsilon$-optimal policy for
$\mathcal{M}_\rdp^\alpha$, then $\pi\alpha$ is an $\epsilon$-optimal policy for
$\rdp$.

\subsection{Related Classes of Decision Processes}

We discuss how Markov abstractions relate to existing classes of decision
processes.

\paragraph{MDPs.}
MDPs can be characterised as the class of NMDPs where histories can be
abstracted into their last observation. Namely, they admit
$\abstracfunc(haor) = o$ as a Markov abstraction.

\paragraph{RDPs.}
A \emph{Regular Decision Process (RDP)} can be defined in terms of
the temporal logic on finite traces $\textsc{ldl}_f$ \cite{brafman2019rdp} or in
terms of finite transducers \cite{ronca2021efficient}.
The former case reduces to the latter by the well-known correspondence between
$\textsc{ldl}_f$ and finite automata.
In terms of finite transducers,
an RDP is an NMDP
$\rdp = \langle A, O, R, \terminationsymbol, \dynfunc \rangle$ whose dynamics
function can be represented by a finite transducer $T$ that,
on every history $h$, outputs the function
$\dynfunc_h: A \leadsto (OR \cup \{ \terminationsymbol \})$
induced by $\dynfunc$ when its first argument is $h$.
Here we observe that the iterated version of the transition function of $T$ is
a Markov abstraction.

\paragraph{POMDPs.}
A \emph{Partially-Observable Markov Decision Process} (POMDP), cf.\
\cite{kaelbling1998pomdp}, is a tuple
$\mathcal{P} = \langle A, O, R, \terminationsymbol, X, \transitionfunc,
  \rewardfunc, \obsfunc, x_0 \rangle$ where:
$A$, $O$, $R$, $\terminationsymbol$ are as in an NMDP;
$X$ is a finite set of \emph{hidden states};
$\transitionfunc: X \times A \leadsto X$ is the \emph{transition
function};
$\rewardfunc: X \times A \times O \leadsto R$ is the \emph{reward function};
$\obsfunc: X \times A \times \leadsto (O \cup \{ \terminationsymbol \})$ is the
\emph{observation function};
$x_0 \in X$ is the initial hidden state.
To define the dynamics function---i.e.,
the function that describes the probability to observe next a certain pair of
observation and reward, or termination, given a certain history of observations
and action---it requires to introduce the \emph{belief function}
$\believefunc: \mathcal{H} \leadsto X$, which describes the probability of being
in a certain hidden state given the current history.
Then,
the dynamics function can be expressed in terms of the belief function as 
$\dynfunc(or|h,a) = \sum_x \believefunc(x|h) \cdot \obsfunc(o|x,a) \cdot
\rewardfunc(r|x,a,o)$ and $\dynfunc(\terminationsymbol|h,a) = \sum_x
\believefunc(x|h) \cdot \obsfunc(\terminationsymbol|x,a)$. Policies and value
functions, and hence the notion of optimality, are as for NMDPs.
For each history $h$, the probability distribution $\believefunc(\cdot|h)$ over
the hidden states is called a \emph{belief state}.
We note the following property of POMDPs.

\begin{restatable}{theorem}{theorempomdp} \label{theorem:pomdp}
  If a POMDP has a finite set of reachable belief states, then the function that
  maps every history to its belief state is a Markov abstraction.
\end{restatable}
\begin{proof}[Proof sketch]
  The function $\alpha(h) = \believefunc(\cdot|h)$ is a Markov abstraction, as
  it can be verified by inspecting the expression of the dynamics function of a
  POMDP given above.
  Specifically, $\alpha(h_1) = \alpha(h_2)$ implies $\believefunc(\cdot|h_1) =
  \believefunc(\cdot|h_2)$, and hence $\dynfunc(\cdot|h_1,a) =
  \dynfunc(\cdot|h_2,a)$ for every action $a$.
\end{proof}

\section{Our Approach to Non-Markov RL}

\begin{figure}[t]
\centering
\vspace{0.1cm}
\includegraphics[width=0.48\textwidth]{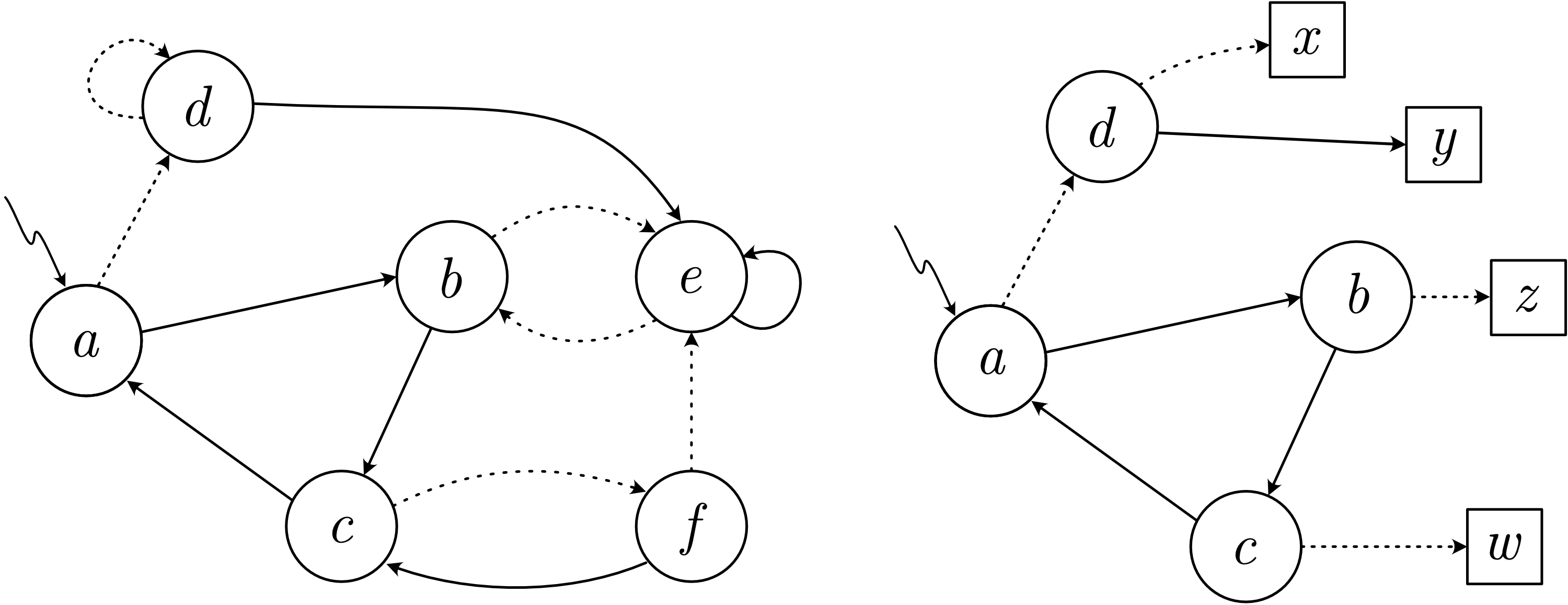}
\caption{Target automaton (left), and hypothesis automaton (right). Solid
edges denote transitions for input $1$, and dashed edges for $0$.}
\label{figure:automata}
\end{figure}

Our approach combines automata learning with Markov RL.
We first describe the two modules separately, and then present the RL algorithm
that combines them.
For the section we fix an NMDP 
$\mathcal{P} = \langle A, O, R, \terminationsymbol, \dynfunc \rangle$, and
assume that it admits a Markov abstraction $\alpha$ on states $S$.

\subsection{First Module: Automata Learning}

Markov abstractions can be learned via automata learning due to the following
theorem.
\begin{restatable}{theorem}{thabstractionautomata} 
  \label{th:abstraction-automata}
  There exist a transition function $\tau: S \times AOR \to S$ and an initial
  state $s_0$ such that, for
  every Markov policy $\pi$ on $S$, the dynamics $\dynfunc_{\pi\alpha}$ of\/
  $\rdp$ are represented by an automaton
  $\langle S, AOR, A\zeta, \tau, \lambda, s_0 \rangle$ for some probability
  function $\lambda$.
  Furthermore, $\tau^* = \alpha$.
\end{restatable}
\begin{proof}[Proof sketch]
  The start state is $s_0 = \alpha(\varepsilon)$.
  The transition function is defined as $\tau(s,aor) = \alpha(h_saor)$ where
  $h_s$ is an arbitrarily chosen history such that $\alpha(h_s) = s$.
  Clearly $\tau^* = \alpha$.
  Then, the probability function is defined as
  $\lambda(s,aor) = \pi(a|s) \cdot \dynfunc(or|h_sa)$ and
  $\lambda(s,a\stopaction) = \pi(a|s) \cdot \dynfunc(\stopaction|h_sa)$.
  It can be shown by induction that the resulting automaton represents
  $\dynfunc_{\pi\alpha}$ regardless of the
  choice of the representative histories $h_s$, since all histories mapped to
  the same state determine the same dynamics function.
\end{proof}

We present an informal description of a generic PDFA-learning algorithm,
capturing the essential features of the algorithms in
\cite{ron1998learnability,clark2004pac,palmer2007pac,balle2013pdfa,balle2014adaptively}.
We will highlight the characteristics that have an impact on the rest of the
RL algorithm.
To help the presentation, consider Figure~\ref{figure:automata}.
The figure shows the transition graph of a target automaton (left), and
the \emph{hypothesis graph} built so far by the algorithm (right).
In the hypothesis graph we distinguish \emph{safe} and \emph{candidate} nodes.
Safe nodes are circles in the figure. They are in a one-to-one correspondence
with nodes in the target, they have all transitions defined, and they are not
going to change.
Candidate nodes are squares in the figure.
Their transitions are not defined, and hence they form the frontier of the
learned part of the graph.
The graph is extended by \emph{promoting} or \emph{merging} a candidate.
If the algorithm tests that a candidate node is distinct from every other
safe state in the hypothesis graph, then it is promoted to safe. 
Upon promotion, a candidate is added for each possible
transition from the just-promoted safe node. This effectively extends the
automaton by pushing the frontier.
If the algorithm tests that a candidate is equivalent to a safe node
already in the hypothesis, then the candidate is merged into the safe node.
The merge amounts to deleting the candidate node and redirecting all its 
incoming edges to the safe node.
Thus, the statistical core of the algorithm consists in the tests. The test
between two nodes is based on the input strings having prefixes that map to the
two nodes respectively. A sufficient number of strings yields a good accuracy of
the test. Assuming that all tests yield the correct result, it is easy to see
that every new hypothesis is closer to the target.
We focus on the approach in \cite{balle2014adaptively}, which has the following
guarantees.

\begin{guarantees} \label{guarantees:automata}
  There are complexity functions $K,N,T_\mathrm{s}$ such that, for every
  $\delta$ and every target automaton $\mathcal{A}$,
  the automata learning algorithm builds a
  sequence of hypotheses $\atm_1, \dots, \atm_n$ satisfying the 
  following conditions:
  \emph{(soundness)}~with probability at least $1-\delta$, 
  for every $i \in [1,n]$, there is a transition-preserving bijection between
  the safe states of $\atm_i$  and a subset of states of the target $\atm$;
  \emph{(incrementality)}~for every $i \in [2,n]$, the hypothesis $\atm_i$
  contains all safe states and all transitions between safe states of
  $\atm_{i-1}$;
  \emph{(liveness)}~every
  candidate $s$ is either promoted or merged within reading an expected number 
  $K(\mathcal{A},\delta)$ of strings having a prefix that maps to
  $s$;
  \emph{(boundedness)}~the number $n$ of hypotheses is at most $N(\mathcal{A})$;
  \emph{(computational cost)}~every string $x$ is processed in time 
  $T_\mathrm{s}(\mathcal{A},\delta, |x|)$.
\end{guarantees}

We will ensure that statistics for a state $q$ are no longer updated once it is
promoted to safe. This preserves the guarantees, and makes the
algorithm insensitive to changes in the distribution $\lambda(q',\cdot)$ of the
target state $q'$ corresponding to $q$ that take place after $q$ is promoted.

Since a hypothesis automaton $\atm_i$ contains candidate states, that do not have 
outgoing transitions, the Markov abstraction $\alpha_i$ obtained as the
iteration $\tau_i^*$ of its transition function is not a complete Markov
abstraction, but a partial one.
As a consequence, the MDP $\mathcal{M}_i$ induced by $\alpha_i$ is also partial.
Specifically, it contains states from which one cannot proceed, the ones
deriving from candidate nodes.

\subsection{Second Module: Markov RL}

When a Markov RL agent is confined to operate in a subset $S'$ of the states $S$
of an MDP, it can always achieve one of two mutually-exclusive results.
The first result is that it can find a policy that is near-optimal for the
entire MDP, ignoring the states not in $S'$.
Otherwise, there is a policy that leads to a state not in $S'$ sufficiently
often.
This is the essence of the Explore or Exploit lemma from \cite{kearns2002near}.
The property is used explicitly in the $E^3$ algorithm \cite{kearns2002near}, 
and more implicitly in algorithms such as RMax \cite{brafman2002rmax}.
Indeed, these algorithms can be employed to satisfy the following guarantees.

\begin{guarantees} \label{guarantees:rl}
  There are complexity functions $E,T_\mathrm{e}$ such that, for every
  $\epsilon$, every $\delta$, and every MDP $\mathcal{M}$ where the agent is
  restricted to operate in a subset $S'$ the states, the two following
  conditions hold:
  \emph{(explore or exploit)}~within
  $E(\mathcal{M},\epsilon,\delta)$ episodes,
  with probability at least $1-\delta$,
  either the agent finds an $\epsilon$-optimal policy or there is a state not in
  $S'$ that is visited at least once;
  \emph{(computational cost)}~every episode $e$ is processed in time
  $T_\mathrm{e}(\mathcal{M},\epsilon,\delta, |e|)$.
\end{guarantees}

\subsection{Overall Approach}

\begin{algorithm}[t]
  \SetKw{Or}{or}%
  \SetKw{And}{and}%
  \SetKw{Yield}{yield}%
  \SetKwFor{Loop}{loop}{}{}%
  \SetKwFor{Until}{until}{do}{}%
  \SetKwInput{Input}{Input}
  \SetKwInput{OInput}{Optional}
  \SetKwInput{Output}{Output}
  \SetKwInput{Parameters}{Parameters}
  \SetKwProg{Fn}{function}{}{}
  $\alpha \leftarrow \emptyset$\;
  \Loop{}{
    $h \leftarrow \varepsilon$\; 
      \While{$\alpha(h)$ is safe \And episode is not over}{
        $a \leftarrow \mathtt{MarkovRL.choose}()$\;
        $o,r \leftarrow$ perform action $a$\;
        $\mathtt{MarkovRL.observe}(\alpha(h),a,r,\alpha(haor))$\;
        $h \leftarrow h aor$;
      }
    \While{episode is not over}{
      $a \leftarrow$ pick an action according to $\pi_\mathrm{u}$\;
      $o,r \leftarrow$ perform action $a$\;
      $h \leftarrow h aor$; 
    }
    $\alpha \leftarrow \mathtt{AutomataLearning.consume}(h)$\;
    $\mathtt{MarkovRL.update}(\alpha)$\;
  }
  \caption{\texttt{NonMarkovRL}}
  \label{alg:abstraction-module}
\end{algorithm}

We will have a Markov agent operating in the MDP induced by the current
partial abstraction, in order to either find a near-optimal policy (when the
abstraction is sufficiently complete) or to visit candidate states (to collect
strings to extend the abstraction).
Concretely, we propose Algorithm~\ref{alg:abstraction-module}, that employs
the modules $\mathtt{AutomataLearning}$ and $\mathtt{MarkovRL}$ to solve
non-Markov RL.
It starts with an empty abstraction $\alpha$ (Line~1), and then it loops,
processing one episode per iteration.
Line~3 corresponds to the beginning of an episode, and hence
the algorithm sets the current history to the empty history.
Lines~4--8 process the prefix of the episode that maps to the safe
states; on this prefix, the Markov RL agent is employed.
Specifically, the agent is first queried for the action to take (Line~5),
the action is then executed (Line~6), the resulting transition is shown to the
agent (Line~7), and the current history is extended (Line~8).
Lines~9--12 process the remainder of the episode, for which we do not
have an abstraction yet.
First, an action is chosen according to the uniform policy $\pi_\mathrm{u}$
(Line~10), the action is then executed (Line~11), and the current history is
extended (Line~12).
Lines~13--14 process the episode that has just been generated, before moving to
the next episode.
In Line~13 the automata learning algorithm is given the episode, and it returns
the possibly-updated abstraction $\alpha$.
Finally, in Line~14, the Markov agent is given the latest abstraction, and it
has to update its model and/or statistics to reflect any change in the
abstraction. We assume a naive implementation of the update function, that
amounts to resetting the Markov agent when the given abstraction is different
from the previous one. 

The algorithm has PAC guarantees assuming PAC guarantees for the employed
modules.
In particular, let $K,N,T_\mathrm{s}$ be as in
Guarantees~\ref{guarantees:automata}, and let $E,T_\mathrm{e}$ be as in
Guarantees~\ref{guarantees:rl}.
In the context of Algorithm~1, the target automaton is the automaton
$\mathcal{A}^\mathrm{u}_\rdp$ that represents the dynamics of $\rdp$
under the uniform policy $\upolicy$, and the MDP the Markov agent interacts
with is the MDP $\mathcal{M}^\alpha_\rdp$ induced by $\alpha$.
Furthermore, let $L_\rdp$ be the expected episode length for $\rdp$.

\begin{theorem}
  For any given $\epsilon$ and $\delta$, and for any NMDP $\rdp$ 
  admitting a Markov abstraction~$\alpha$,
  Algorithm~1 has probability at least $1-\delta$ of solving the Episodic RL
  problem for $\rdp$ and~$\epsilon$, using a number of episodes
  $$O\big(N(\mathcal{A}^\mathrm{u}_\rdp) \cdot
       \left\lceil
      E(\mathcal{M}_\rdp^\alpha,\epsilon,\delta')/K(\mathcal{A}^\mathrm{u}_\rdp,\delta/2)
  \right\rceil\big)$$
    with $\delta' = \delta/(2 \cdot N(\mathcal{A}^\mathrm{u}_\rdp))$,
    and a number of computation steps that is proportional to the number of
    episodes by a quantity
    $$
    O\big(T_\mathrm{s}(\mathcal{A}^\mathrm{u}_\rdp,\delta/2,L_\rdp) + 
    T_\mathrm{e}(\mathcal{M}_\rdp^\alpha,\epsilon,\delta',L_\rdp)\big).
    $$
\end{theorem}
\begin{proof}[Proof sketch]
  An execution of the algorithm can be split into stages, with one stage for
  each hypothesis automaton. By the boundedness condition, there are at most
  $N(\mathcal{A}^\mathrm{u}_\rdp)$ stages.
  Consider an arbitrary stage. Within
  $E(\mathcal{M}_\rdp^\alpha,\epsilon,\delta')$ episodes, the Markov agent
  either exploits or explores.
  If it exploits, i.e., it finds an $\epsilon$-optimal policy $\pi$ for
  the MDP induced by the current partial abstraction $\alpha_i$, then
  $\pi\alpha_i$ is $\epsilon$-optimal for $\rdp$.
  Otherwise, a candidate state is explored
  $K(\mathcal{A}^\mathrm{u}_\rdp,\delta/2)$ times, and hence it is promoted or
  merged. 
  The probability $\delta$ of failing is partitioned among the one run of
  the automata learning algorithm, and the $N(\mathcal{A}^\mathrm{u}_\rdp)$
  independent runs of the Markov RL algorithm.
  Regarding the computation time, note that the algorithm performs one
  iteration per episode, it operates on one episode in each iteration, calling
  the two modules and performing some other operations having a smaller cost.
\end{proof}

\begin{figure*}[t]
    \centering
    \begin{subfigure}[b]{0.24\textwidth}
        \centering
        \includegraphics[width=\textwidth]{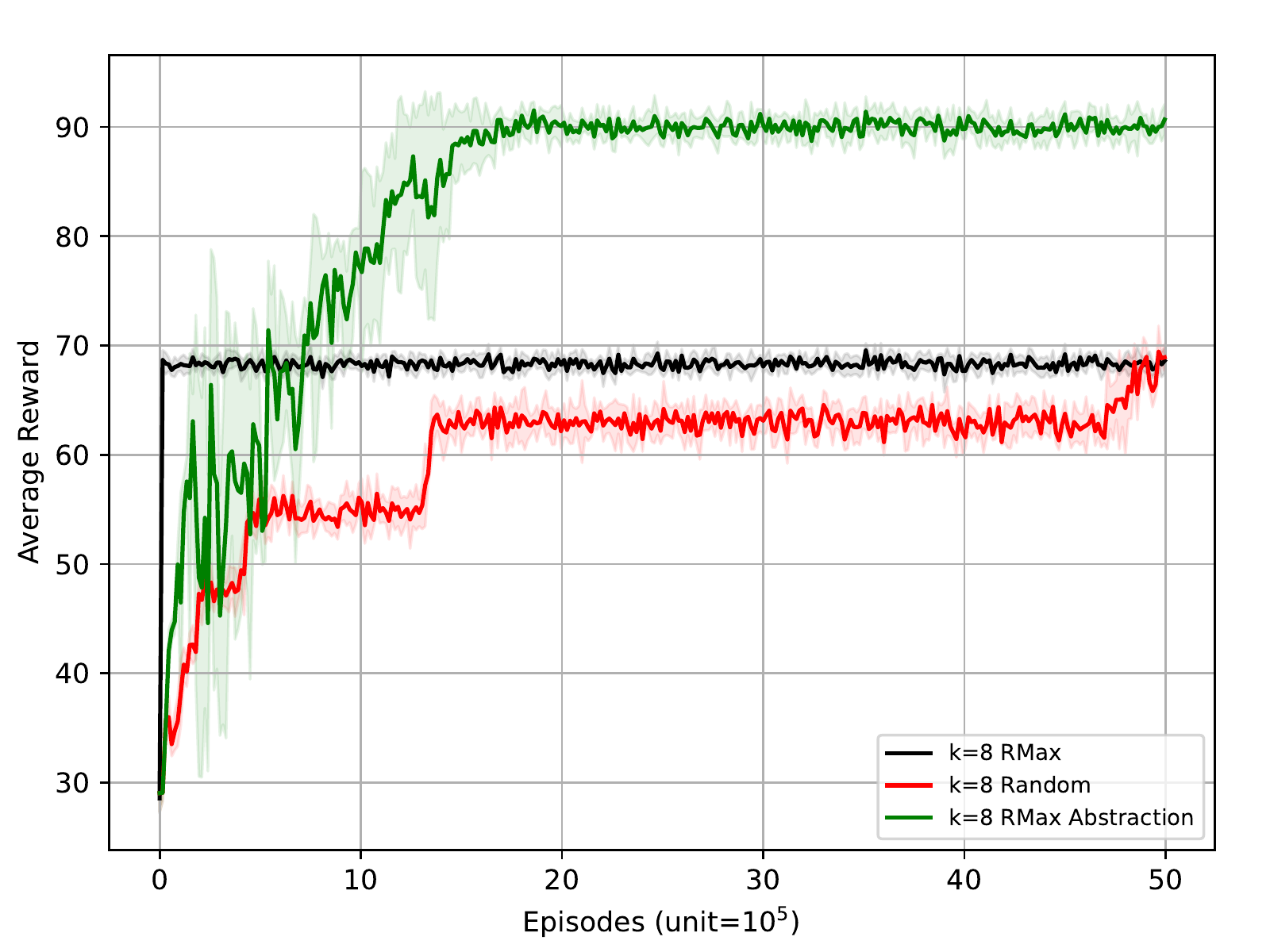}
        \caption{Reset-Rotating MAB}
        \label{fig:rotating_mab_v2}
    \end{subfigure}
    \begin{subfigure}[b]{0.24\textwidth}
        \centering
        \includegraphics[width=\textwidth]{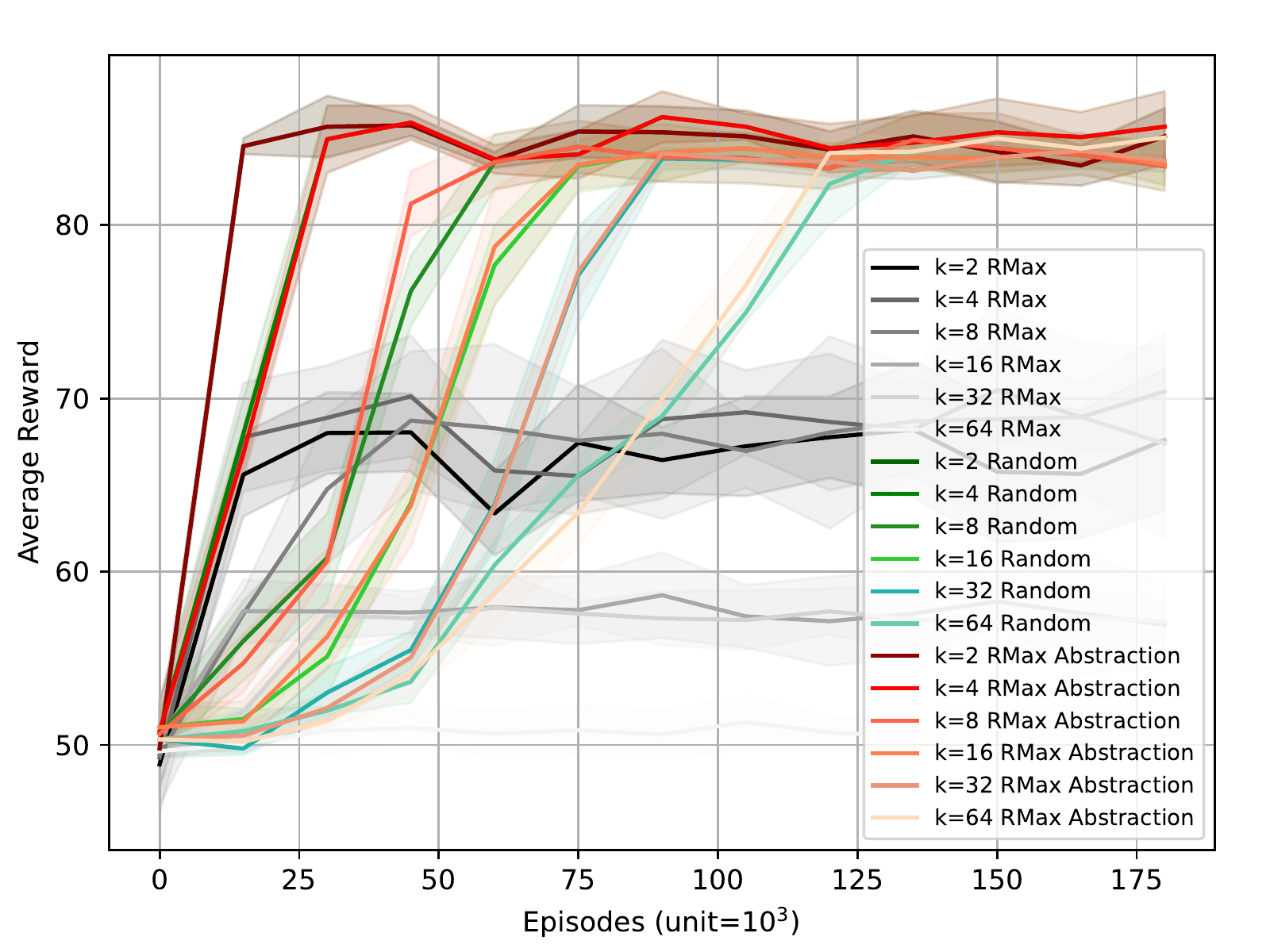}
        \caption{Enemy Corridor}
        \label{fig:enemy_corridor}
    \end{subfigure}
    \begin{subfigure}[b]{0.24\textwidth}
        \centering
        \includegraphics[width=\textwidth]{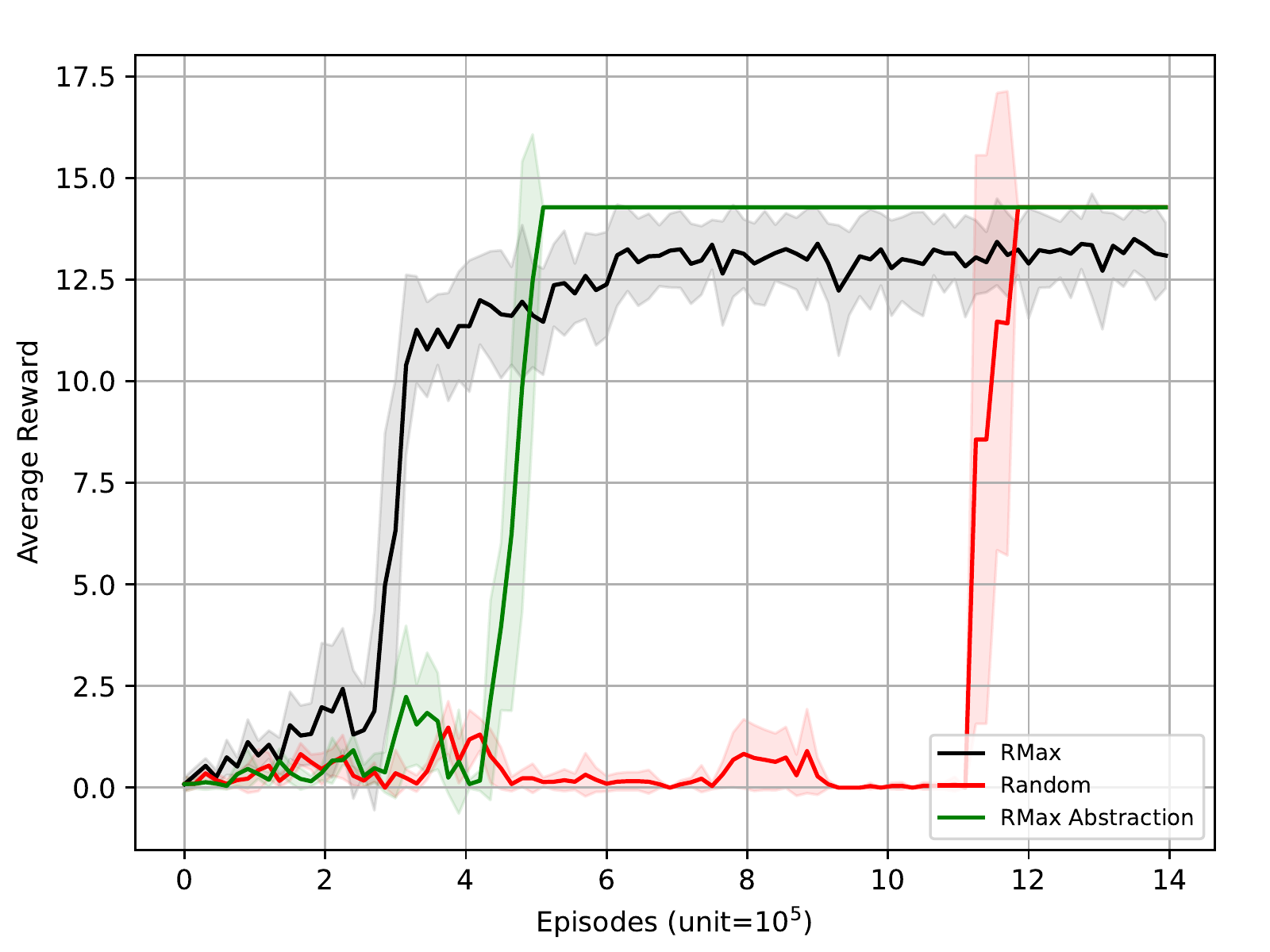}
        \caption{Flickering Grid}
        \label{fig:flickering_grid}
    \end{subfigure}
    \begin{subfigure}[b]{0.24\textwidth}
        \centering
        \includegraphics[width=\textwidth]{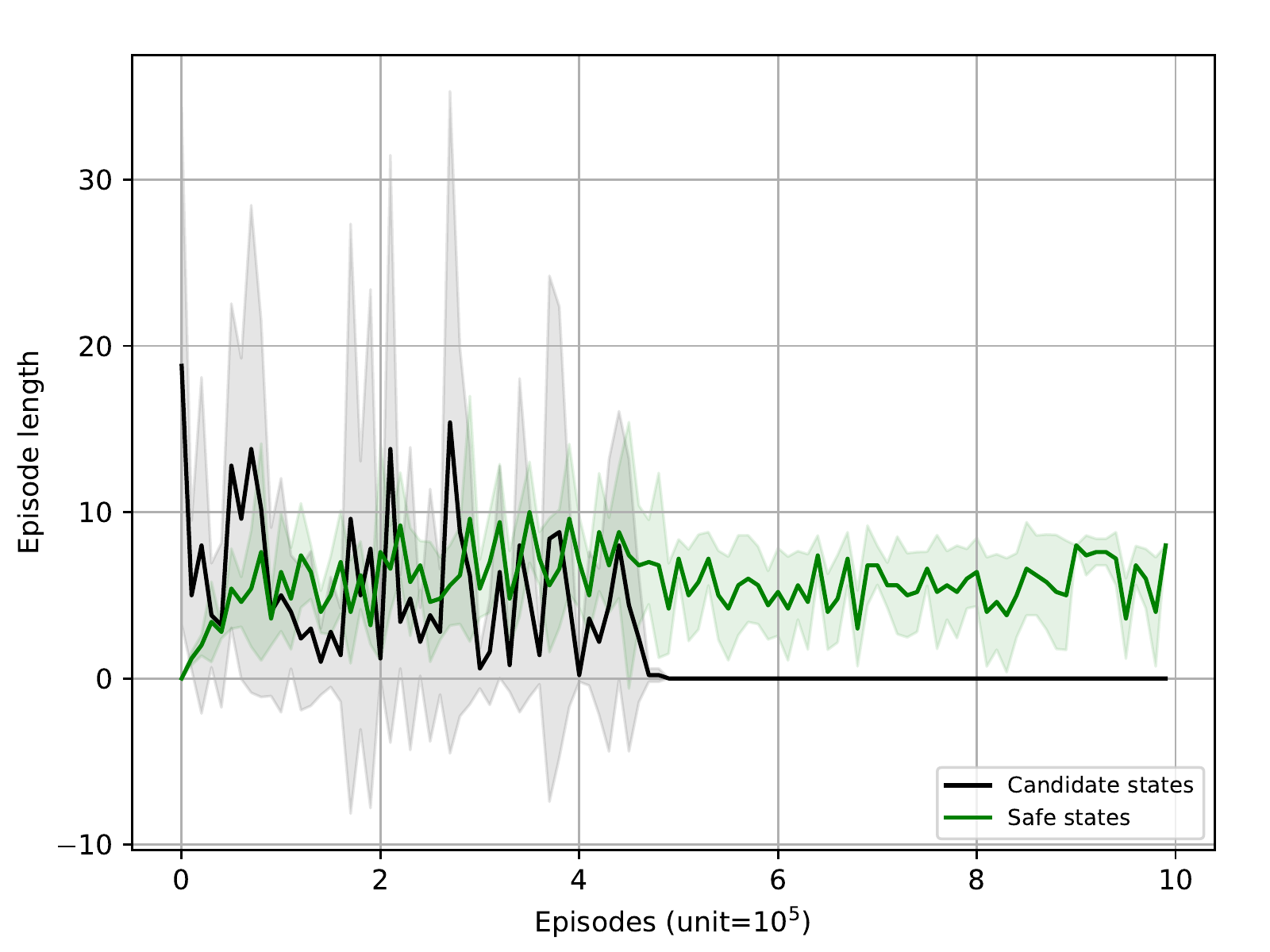}
        \caption{Safe vs.\ candidate states}
        \label{fig:safe_vs_candidate}
    \end{subfigure}
    \begin{subfigure}[b]{0.24\textwidth}
        \centering
        \includegraphics[width=\textwidth]{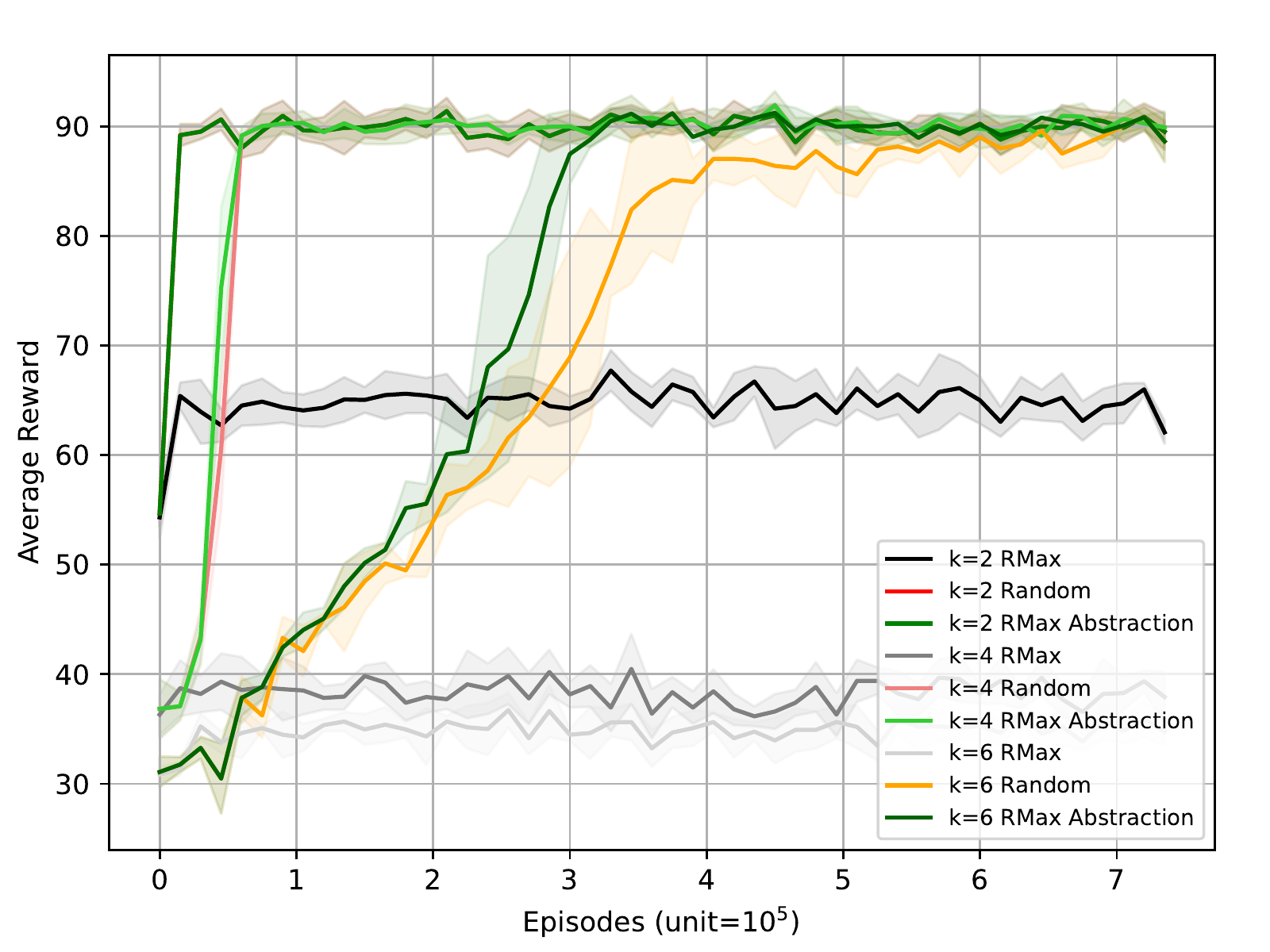}
        \caption{Rotating MAB}
        \label{fig:rotating_mab}
    \end{subfigure}
    \begin{subfigure}[b]{0.24\textwidth}
        \centering
        \includegraphics[width=\textwidth]{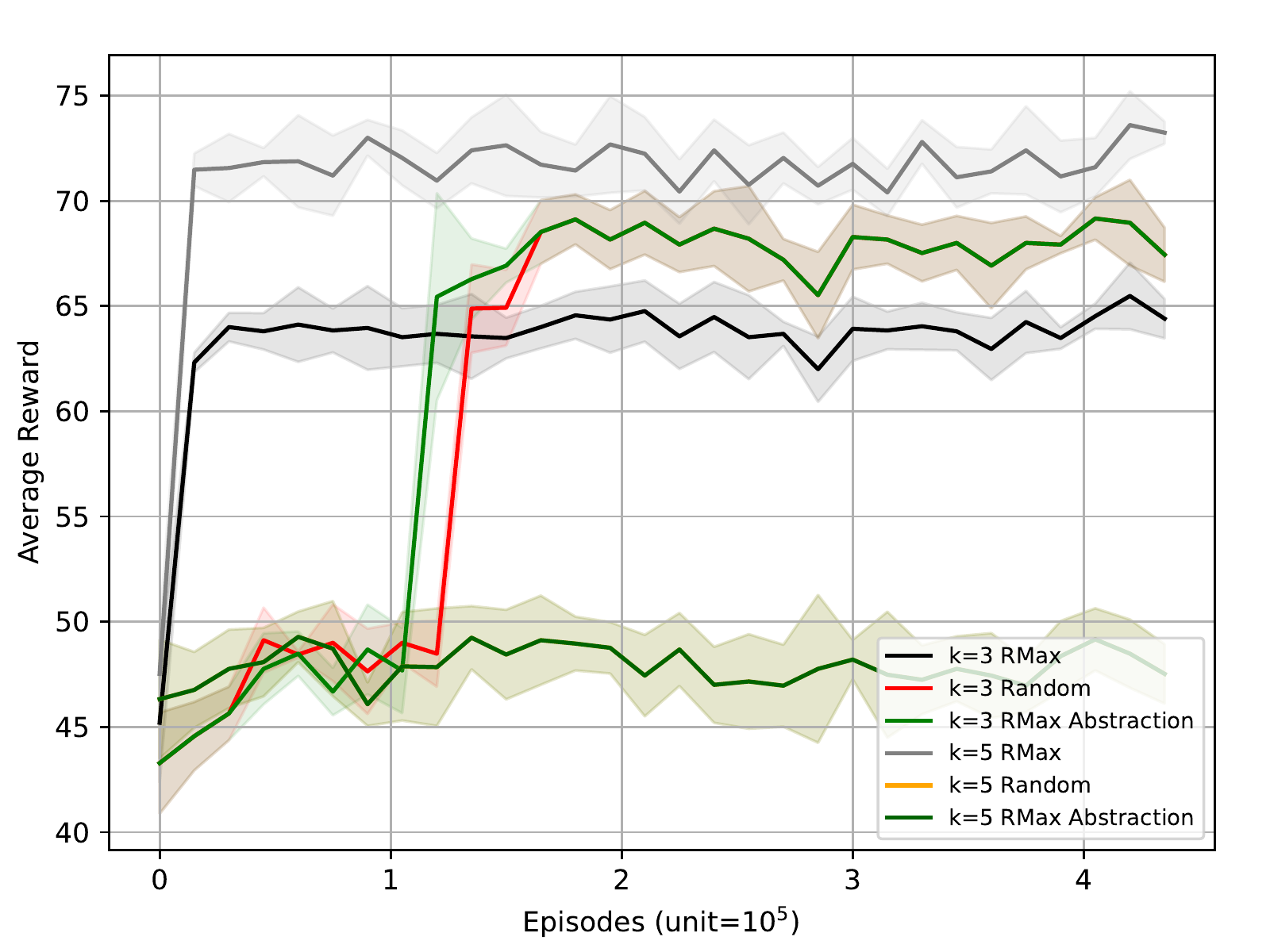}
        \caption{Malfunction MAB}
        \label{fig:malfunction_mab}
    \end{subfigure}
    \begin{subfigure}[b]{0.24\textwidth}
        \centering
        \includegraphics[width=\textwidth]{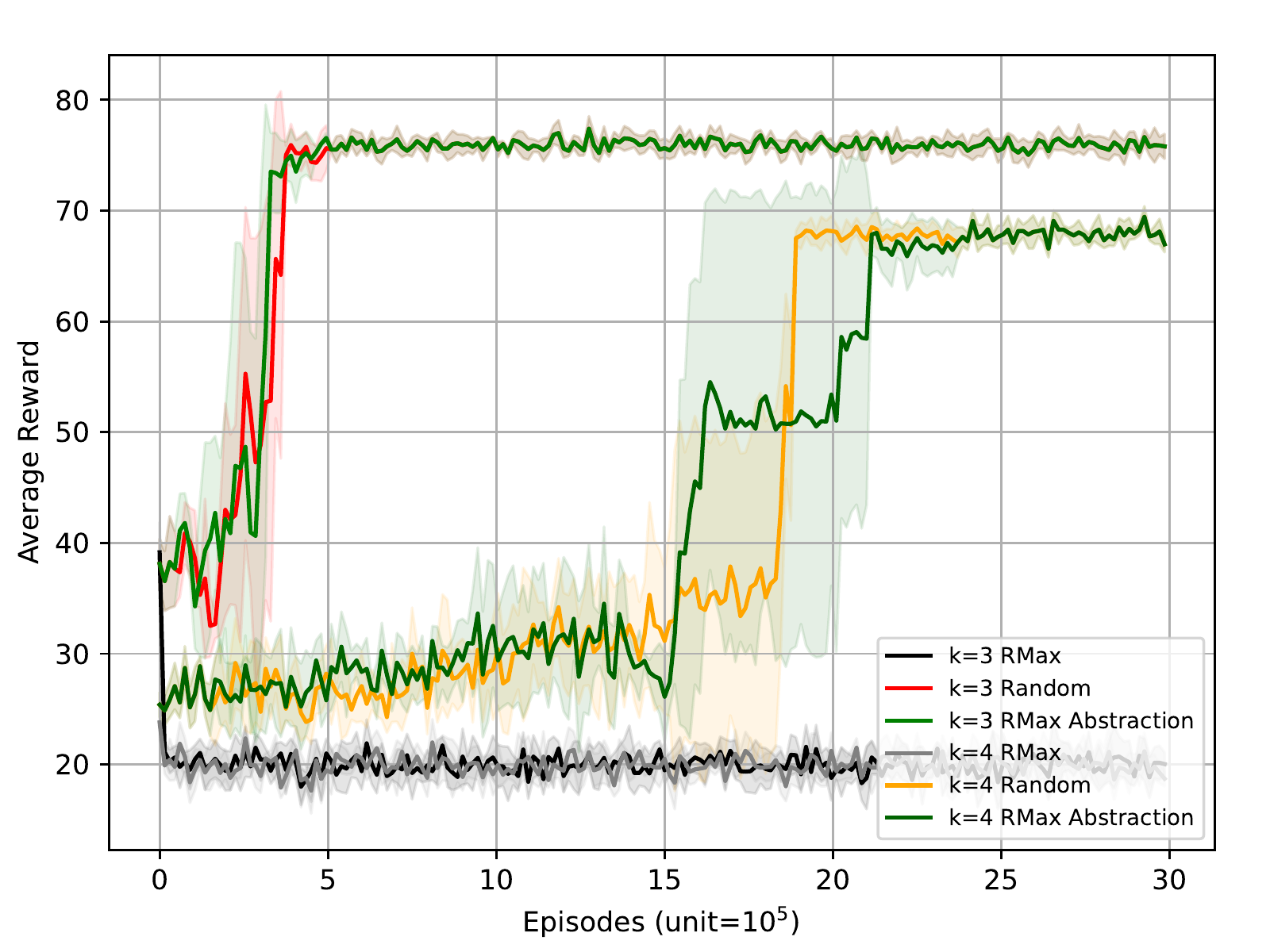}
        \caption{Cheat MAB}
        \label{fig:cheat_mab}
    \end{subfigure}
    \begin{subfigure}[b]{0.24\textwidth}
        \centering
        \includegraphics[width=\textwidth]{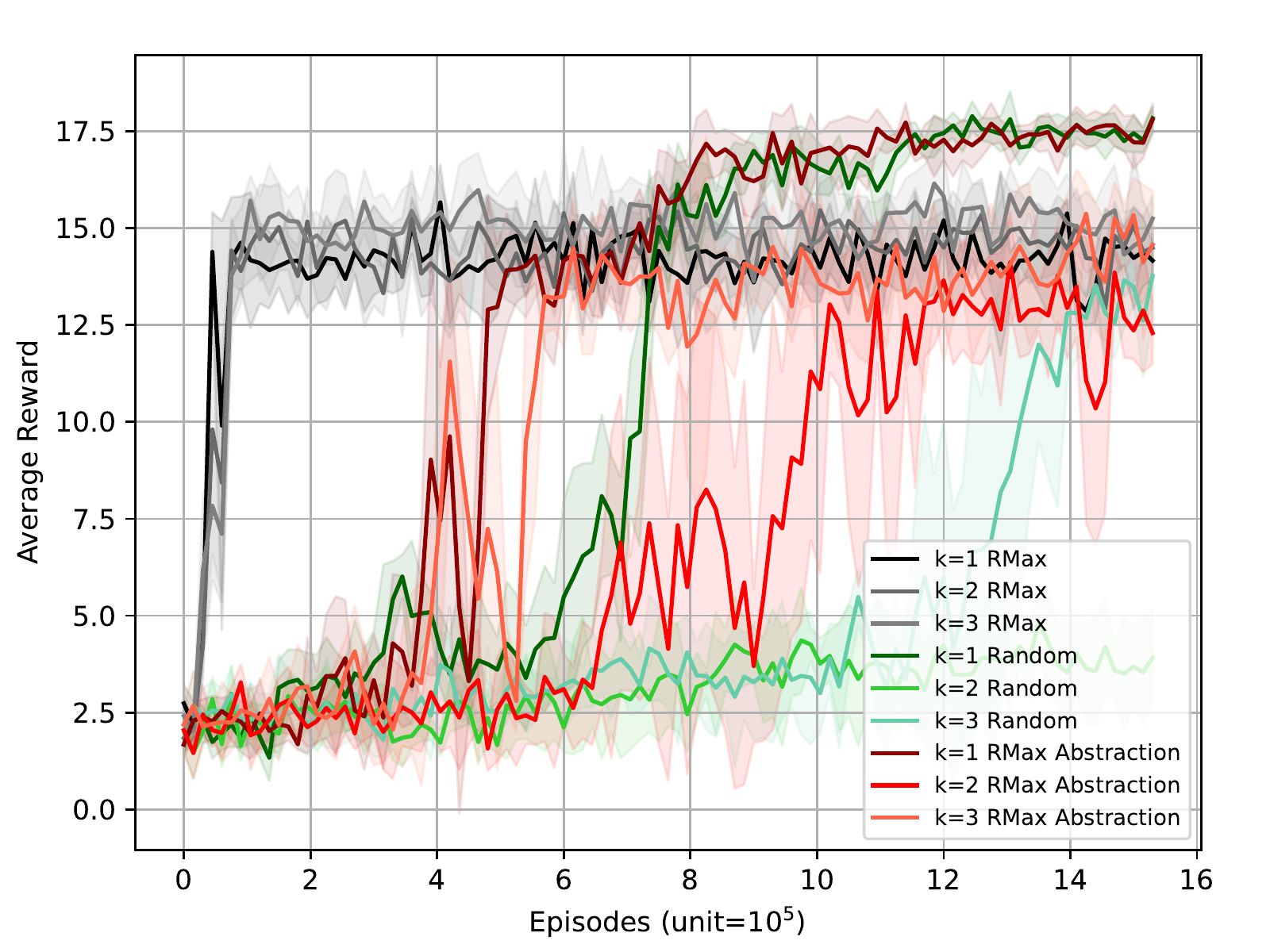}
        \caption{Rotating Maze}
        \label{fig:rotating_maze}
    \end{subfigure}
    \caption{Empirical evaluation of Markov Abstractions.}
    \label{fig:evaluation}
\end{figure*}

\section{Empirical Evaluation}



We show an experimental evaluation of our approach.%
\footnote{The results shown in this paper are reproducible. Source code,
  instructions, and definitions of the experiments are available at:
  \href{https://github.com/whitemech/markov-abstractions-code-ijcai22}{github.com/whitemech/markov-abstractions-code-ijcai22}. 
Experiments were carried out on a server running Ubuntu 18.04.5 LTS, with 512GB RAM, 
and an 80 core Intel Xeon E5-2698 2.20GHz.
Each training run takes one core. The necessary amount of compute and time
depends mostly on the number of episodes set for training.}
We employ the state-of-the-art stream PDFA learning
algorithm~\cite{balle2014adaptively} and the classic RMax
algorithm~\cite{brafman2002rmax}.
They satisfy Guarantees~\ref{guarantees:automata} and~\ref{guarantees:rl}.
%
We consider the domains from \cite{abadi2020learning}: Rotating MAB, Malfunction
MAB, Cheat MAB, and Rotating Maze;
a variant of Enemy Corridor \cite{ronca2021efficient};
and two novel domains: Reset-Rotating MAB, and Flickering Grid. 
Reset-Rotating MAB is a variant of Rotating MAB where failing to pull the
correct arm brings the agent back to the initial state. 
Flickering Grid is a basic grid with an initial and a goal position, but where
at some random steps the agent is unable to observe its position.
All domains are parametrised by $k$, which makes the
domains more complex as it is assigned larger values.

Figure~\ref{fig:evaluation} shows our performance evaluation. 
Plots include three different approaches. 
First, our approach referred to as \emph{RMax Abstraction}, i.e., the RMax
Markov agent provided with partial Markov abstractions as they are incrementally
learned.
Second, a \emph{Random Sampling} approach, equivalent to what is proposed in
\cite{ronca2021efficient}, that always explores at random.
Third, the RMax agent as a baseline for the performance of a purely Markov
agent, that does not rely on Markov abstractions.
Results for each approach are averaged over 5 trainings and show the standard deviation. 
At each training, the agent is evaluated at every 15k training episodes.
Each evaluation is an average over 50 episodes, where for each episode we measure the accumulated reward
divided by the number of steps taken. 
Notice that the RMax Abstraction agent takes actions uniformly at random on histories where the Markov abstraction is not yet defined,
otherwise taking actions greedily according to the value function. 

The results for Reset-Rotating MAB (Figure~\ref{fig:rotating_mab_v2}) show the
advantage of the exploration strategy of our approach, compared to Random
Sampling. In fact, in this domain, random exploration does not allow for
exploring states that are far from the initial state.
The results for Enemy Corridor (Figure~\ref{fig:enemy_corridor})
show that our approach scales with the domain size. Namely, in Enemy Corridor
we are able to increase the corridor size $k$
up to $64$.
The results for Flickering Grid (Figure~\ref{fig:flickering_grid}) show that our
approach works in a domain with partial observability, that is natural to model
as a POMDP. This is in line with our Theorem~\ref{theorem:pomdp}. The
Markov abstraction for Flickering Grid maps histories to the current cell;
intuitively, the agent learns to compute its position even if sometimes it is
not explicitly provided by the environment.
For this domain, Figure~\ref{fig:safe_vs_candidate} shows the number of steps
the agent is in a safe state against when it is not.
The figure exemplifies the incremental process of learning a Markov abstraction,
as the time spent by the agent out of the safe region is high during the first
episodes, and it decreases as more histories are collected and more safe states
get to be learned by the stream PDFA learning algorithm;
we see the agent achieves optimal behaviour around $5 \cdot 10^5$ episodes, when 
it no longer encounters candidate states, since all relevant states have been
learned.

In the domains from \cite{abadi2020learning}
(Figures~\ref{fig:rotating_mab}-\ref{fig:rotating_maze}) our approach has
poor performance overall. We are able to show convergence in all domains except
Rotating Maze with $k = 2,3$ and Malfunction MAB with $k = 5$. However, it
requires a large number of episodes before achieving a close-to-optimal average
reward. 
We observe that all these domains have a low distinguishability between the
probability distribution of the states. The performance of PDFA-learning is
greatly dependent on such parameter, cf.\ \cite{balle2013pdfa}.
Furthermore, in these domains, the value of distinguishability decreases with
$k$.
On the contrary, in Reset-Rotating MAB, Enemy Corridor, and Flickering Grid,
distinguishability is high and independent of $k$.
Therefore, we conclude that a low distinguishability is the source of poor
performance. In particular, it causes the automata learning module to require a
large number of samples.
We compare our approach against two related approaches.
First, we compare against \cite{abadi2020learning} that combines
deterministic finite automata learning with history clustering to build a
Mealy Machine that represents the underlying decision process, 
while employing MCTS to compute policies for the learned models. 
Comparing to the results published in \cite{abadi2020learning} for Rotating MAB,
Malfuction MAB, and Rotating Maze, our approach has better performance in
Rotating MAB, and worse performance in the other domains. 
Second, we compare against \cite{toroicarte2019learning}, also based on automata
learning.
Our approach outperforms theirs in all the domains mentioned in this paper, and
their approach outperforms ours in all their domains. Specifically, their
algorithm does not converge in our domains, and our algorithm does not converge
in theirs. Our poor
performance is explained by the fact that all domains in
\cite{toroicarte2019learning} have a low distinguishability, unfavourable to us.
On the other hand, their poor performance in our domains might be due to their
reliance on local search methods.

\section{Conclusion and Future Work}

We presented an approach for learning and solving non-Markov decision processes,
that combines standard RL with automata learning in a modular manner.
Our theoretical results show that, for NMDPs that admit a Markov 
abstraction, $\epsilon$-optimal policies can be computed  
with PAC guarantees, and that its dynamics can be accurately represented by an
automaton.
The experimental evaluation shows the feasibility of the approach.
The comparison with the random-exploration baseline shows the
advantage of a smarter exploration based on the automaton as it is still being
learned, which is an important aspect introduced in this paper.
Some of the domains illustrate the difficulty of learning when
distinguishability is low---an issue of the underlying automata learning
algorithm. At the same time, other domains show good scalability of the
overall approach when distinguishability stays high.

The difficulties encountered by the PDFA-learning algorithm in domains with
low distinguishability suggest that the existing
PDFA-learning techniques should be further developed in order to be effective in
RL. We see it as an interesting direction for future work.
A second direction is to make our approach less reliant on the
automata learning component. Specifically, an agent could operate directly on
the histories for which the current Markov abstraction is yet undefined.

\section*{Acknowledgments}
This work is partially supported by the ERC Advanced Grant WhiteMech (No.\
834228), by the EU ICT-48 2020 project TAILOR (No.\ 952215), by the PRIN project
RIPER (No.\ 20203FFYLK), and by the JPMorgan AI Faculty Research Award
``Resilience-based Generalized Planning and Strategic Reasoning''.

\bibliographystyle{named}
\bibliography{bibliography,journal-abbreviations}

  \newpage
  \onecolumn
  \appendix

\clearpage
\section{Proofs}
\label{sec:proofs}

\thabstractionautomata*
\begin{proof}
  The start state is $s_0 = \alpha(\varepsilon)$.
  The transition function is defined as $\tau(s,aor) = \alpha(h_saor)$ where
  $h_s$ is an arbitrarily chosen history such that $\alpha(h_s) = s$.
  Clearly $\tau^* = \alpha$.
  Then, the probability function is defined as
  $\lambda(s,aor) = \pi(a|s) \cdot \dynfunc(or|h_sa)$ and
  $\lambda(s,a\stopaction) = \pi(a|s) \cdot \dynfunc(\stopaction|h_sa)$
  Then, the automaton represents the probability
  $\atm(aor e | h) = \lambda(\tau(h),aor) \cdot \atm(e | haor)$ 
  with base case $\atm(a\terminationsymbol | h) = 
  \lambda(\tau(h), a\terminationsymbol)$.
  We show by induction that the automaton represents $\dynfunc_{\pi\alpha}$.

  In the base case we have 
  \begin{align}
    \label{eq:thabstractionautomata-1}
    & \atm(a\terminationsymbol | h) 
    \\
    \label{eq:thabstractionautomata-2}
    & = \lambda(\tau(h), a\terminationsymbol) 
    \\
    \label{eq:thabstractionautomata-3}
    & = \pi(a|\tau(h)) \cdot \dynfunc(\stopaction|h_{\tau(h)},a) 
    \\
    \label{eq:thabstractionautomata-4}
    & = \pi(a|\tau(h)) \cdot \dynfunc(\stopaction|h,a) 
    \\
    \label{eq:thabstractionautomata-5}
    & = \dynfunc_{\pi\alpha}(a\stopaction|h).
  \end{align}
  Note that \eqref{eq:thabstractionautomata-2} and
  \eqref{eq:thabstractionautomata-3} hold by the definition, 
  \eqref{eq:thabstractionautomata-4} by the definition of Markov abstraction,
  and \eqref{eq:thabstractionautomata-4} holds by the definition of dynamics.
  The inductive case follows similarly, 
  \begin{align}
    \label{eq:thabstractionautomata-6}
    & \atm(aor e | h) 
    \\
    \label{eq:thabstractionautomata-7}
    & = \lambda(\tau(h),aor) \cdot \atm(e | haor)
    \\
    \label{eq:thabstractionautomata-8}
    & = \pi(a|\tau(h)) \cdot \dynfunc(or|h_{\tau(h)},a) \cdot \atm(e | haor)
    \\
    \label{eq:thabstractionautomata-9}
    & = \pi(a|\tau(h)) \cdot \dynfunc(or|h,a) \cdot \atm(e | haor)
    \\
    \label{eq:thabstractionautomata-10}
    & = \pi(a|\tau(h)) \cdot \dynfunc(or|h,a) \cdot \dynfunc_{\pi\alpha}(e | haor)
    \\
    \label{eq:thabstractionautomata-11}
    & = \dynfunc_{\pi\alpha}(aore | h).
  \end{align}
  The only difference with the previous case is 
  \eqref{eq:thabstractionautomata-10} where we apply the inductive hypothesis.
\end{proof}

\clearpage
\section{Domains}
\label{sec:domains}

We provide a description of each domain used in the experimental evaluation.

\paragraph{Rotating MAB.} 
This domain was introduced in \cite{abadi2020learning}.
It is a MAB with $k$ arms where reward probabilities depend
on the history of past rewards: they are shifted ($+1 \operatorname{mod} k$)
every time the agent obtains a reward.The optimal behaviour is not only to pull the arm that has the highest reward probability, but to 
consider the shifted probabilities to the next arm once the agent is rewarded, and therefore pull the 
next arm.
Episodes in this domain are set to 10 steps.

\paragraph{Reset-Rotating MAB.}
Similar to the definition above, however reward probabilities are reset to their 
initial state every time the agent does not get rewarded.
Episodes in this domain are set to 10 steps.

\paragraph{Malfunction MAB.} 
This domain was introduced in \cite{abadi2020learning}.
It is a MAB in which one arm has the highest probability of reward, but yields reward zero for one
turn every time it is pulled $k$ times.
Once this arm is pulled $k$ times, it has zero probability of reward in the next step.
The optimal behaviour in this domain is not only to pull the arm with highest reward probability, 
but to take into consideration the number of times that the arm was pulled, such that once the arm is broken 
the optimal arm to pull in the next step is any other arm with nonzero probability of reward.
Episodes in this domain are set to 10 steps.

\paragraph{Cheat MAB.} 
This domain was introduced in \cite{abadi2020learning}.
It is a standard MAB. However, there is a \emph{cheat} sequence of $k$ actions
that, once performed, allows maximum reward at every subsequent step.
The cheat is a specific sequence $[a_{i}, a_{i+1}, ..., a_{k}]$ of arms that have to be pulled.
Observations in this domain are the arms pulled by the agent.
The optimal behaviour is to perform the actions that compose the cheat, and 
thereafter any action returns maximum reward. 
Episodes in this domain are set to 10 steps.

\paragraph{Rotating Maze.} 
This domain was introduced in \cite{abadi2020learning}.
It is a grid domain with a fixed goal position that is $5$ steps away from the
initial position.
The agent is able to move in any direction (up/left/down/right) and the actions
have success probability $0.9$, with the agent moving into the opposite
direction when an action fails.
Every $k$ actions, the orientation of the agent is changed by $90$ degrees
counter-clockwise. 
Observations in this domain are limited to the coordinates of the grid. 
Episodes in this domain are set to 15 steps.

\paragraph{Flickering Grid.}
An 8x8 grid domain ($k=8$) with goal cell $(3,4)$, where at each step the agent observes
a flicker (i.e. a blank observation) with $0.2$ probability.
Episodes in this domain are set to 15 steps. 

\paragraph{Enemy Corridor.}
This is an adapted version of the domain introduced in \cite{ronca2021efficient}.
It is a domain in which the agent has to cross a $2 \times k$ grid while avoiding enemies.
Every enemy guards a two-cell column, and it is in either one of the two cells according
to a certain probability. Such probabilities are swapped every time the agent hits an enemy.
Furthermore, once the agent crosses half of the corridor, the enemies change their set of 
probabilities and guard cells according to a different strategy than the first half of the corridor.
Episodes in this domain are set according to the corridor size (check the source code for exact numbers).

  \newpage


\end{document}